\documentclass[journal,12pt,onecolumn]{IEEEtran}
\usepackage{amsmath,amsfonts,amssymb,float,marginnote,amsthm}
\usepackage{todonotes}
\usepackage{subcaption}
\usepackage{booktabs}
\usepackage{ dsfont }
\usepackage{graphicx}
\usepackage{epsfig}
\usepackage{color, lipsum}
\usepackage{tikz, tikz-qtree}
\usepackage{multicol, multirow}
\usepackage{algorithm}
\usepackage[noend]{algpseudocode}
\usepackage{url}

\newtheorem{theorem}{Theorem}
\newtheorem{lemma}{Lemma}
\newtheorem{definition}{Definition}

\newtheorem{remark}{Remark}

\newcommand{\ANM} [1]{}

\def\ddeg{d_{deg}}

\DeclareGraphicsExtensions{.png}
\DeclareGraphicsExtensions{.ps}

\usepackage{lipsum}
\usepackage{amsfonts}
\usepackage{graphicx}
\usepackage{epstopdf}
\ifpdf
  \DeclareGraphicsExtensions{.eps,.pdf,.png,.jpg}
\else
  \DeclareGraphicsExtensions{.eps}
\fi

\usepackage{enumitem}
\setlist[enumerate]{leftmargin=.5in}
\setlist[itemize]{leftmargin=.5in}

\title{Fundamental Limits of Deep Graph Convolutional Networks\thanks{This paper was presented in part at the 2020 IEEE Symposium on Information Theory~\cite{magnerISIT2020}.}}

\author{
\IEEEauthorblockN{Abram Magner, Mayank Baranwal and
Alfred O. Hero III} 
\thanks{

Abram Magner (email: amagner@albany.edu) is with the Department of Computer Science, University at Albany, SUNY, Albany, NY, USA.  Mayank Baranwal (email: mayankb@umich.edu) and Alfred Hero (email: hero@eecs.umich.edu) are with the Department of EECS, University of Michigan, Ann Arbor, MI, USA.

This research was partially supported by grants from ARO W911NF-19-1026, ARO W911NF-15-1-0479, and ARO W911NF-14-1-0359 and the  Blue Sky Initiative from the College of Engineering at the University of Michigan.
}}

\usepackage{amsopn}

\def\Ind{\mathds{I}}

\def\Pr{{\mathbb{P}}}
\def\E{{\mathbb{E}}}

\def\Gr{\mathcal{G}}

\def\R{{\mathbb{R}}}

\def\N{{\mathbb{N}}}

\def\Bernoulli{\mathrm{Bernoulli}}
\def\epsres{\epsilon_{res}}

\def\symmdiff{\triangle}
\def\Vol{\mathrm{Vol}}
\def\Cu{\mathbb{H}}

\def\SBM{\mathrm{SBM}}
\def\Param{\mathcal{P}}
\def\One{\mathbf{1}}
\def\Act{\mathcal{A}}

\def\lintersect{\cap}
\def\lunion{\cup}

\algnewcommand{\LeftComment}[1]{\Statex \(\triangleright\) #1}

\DeclareGraphicsExtensions{.pdf}

\begin{document}
\maketitle
\thispagestyle{plain}
%
%


\begin{abstract}
Graph convolutional networks (GCNs) are a widely used method for
graph representation learning.  
To elucidate the capabilities and limitations of GCNs, we investigate their power, as a function of their number of
layers, to distinguish between different
random graph models (corresponding to different class-conditional distributions in a classification problem) 
on the basis of the embeddings of their sample
graphs.  In particular, the graph models that we consider arise from
graphons, which are the most general possible parameterizations of
infinite exchangeable graph models and which are the central objects of study
in the theory of dense graph limits.  We give a precise characterization of the
set of pairs of graphons that are indistinguishable by a GCN with nonlinear activation
functions coming from a certain broad class if its depth is at least logarithmic
in the size of the sample graph.  This characterization is in terms of a degree
profile closeness property.  Outside this class, a very simple GCN architecture
suffices for distinguishability.  We then exhibit a concrete, infinite class
of graphons arising from stochastic block models that are well-separated in terms of cut distance and are indistinguishable by a GCN. 
These results theoretically match empirical observations of several
prior works.  
To prove our results, we exploit a connection to random
walks on graphs.  Finally, we give empirical results on synthetic and real graph
classification datasets, indicating that indistinguishable graph distributions
arise in practice.

\end{abstract}

\section{Introduction}
\label{Introduction}

In applications ranging from drug discovery and design \cite{sun2019graph} to proteomics \cite{randic2002comparative} to neuroscience \cite{sporns2003graph} to social network analysis \cite{barnes1983graph}, inputs to machine learning methods often
take the form of graphs.  In order to leverage the empirical success of deep
learning and other methods that work on vectors in finite-dimensional Euclidean spaces for supervised learning tasks in this domain, a plethora of graph representation learning schemes have
been proposed and used~\cite{Hamilton2017RepresentationLO}.
Among these, one method is the \emph{graph convolutional network}
(GCN) architecture~\cite{kipfwelling,vandergheynst}.  
A graph convolutional network works by associating with each node of an input
graph a vector of features and passing these node features through a sequence of
\emph{layers}, resulting in a final set of node vectors, called node embeddings.
To generate a vector representing the entire graph, these final embeddings are 
sometimes averaged.  Each layer of the network consists of a graph diffusion step,
where a node's feature vector is averaged with those of its neighbors; a feature
transformation step, where each node's vector is transformed by a weight matrix;
and, finally, application of an elementwise nonlinearity such as the ReLU or
sigmoid activation function.  The weight matrices are trained from data, so that the metric 
structure of the resulting
embeddings are (one hopes) tailored to a particular classification task.

While GCNs and other graph representation learning methods have been successful
in practice, numerous theoretical questions about their capabilities and the
roles of their hyperparameters remain.  In this paper, we establish fundamental limits on the
ability of GCNs to distinguish between classes of graphs.
We focus on the roles that the number of layers and the presence or absence
of nonlinearity play.  These results are obtained using random graph models; instances of graphs from each class are assumed to be random
samples from a distribution, i.e., a model for each graph class.
The random graph models that we consider are those that
are parameterized by \emph{graphons}~\cite{lovaszbook}, which are functions from
the unit square to the interval $[0, 1]$ that essentially encode edge density
among a continuum of vertices.  Graphons are the central objects of study in the
theory of dense graph limits and, by the Aldous-Hoover theorem~\cite{aldous-exchangeable} exactly parameterize the class of infinite
exchangeable 
random graph models -- those models whose samples are invariant in distribution 
under permutation of vertices.

\subsection{Prior Work}
A survey of modern graph representation learning methods is provided in~\cite{Hamilton2017RepresentationLO}.
Graph convolutional networks were first introduced in~\cite{vandergheynst}, and since then, many variants have been proposed.
For instance, the polynomial convolutional filters in the original work
were replaced by linear convolutions~\cite{kipfwelling}. Authors in \cite{ruiz2019gated} modified the original architecture to include gated recurrent units for working with dynamical graphs. These and other variants have been used in various applications, e.g.,~\cite{jepsen2019graph,coley2019graph,yao2019experimental,duvenaud2015convolutional}.

Theoretical work on GCNs has been from a variety of perspectives.
In~\cite{verma2019stability}, the authors investigated the generalization and stability properties of GCNs.  Several works,
including~\cite{morris2019weisfeiler, chen2019equivalence, xu2018powerful}, have drawn connections between the representation
capabilities of GCNs and the distinguishing ability of the \emph{Weisfeiler-Lehman} (WL) algorithm for graph isomorphism testing~\cite{Weisfeiler1968ReductionOA}.  These
papers drawing comparisons to the WL algorithm implicitly study the injectivity
properties of the mapping from graphs to vectors induced by GCNs.  However, they
do not address the metric/analytic properties, which are important in consideration
of their performance as representation learning methods~\cite{arorarepresentationlearning}.  Finally, at least one work has considered the performance of untrained GCNs on community detection~\cite{Kawamoto2018}. { The authors of that paper provide a heuristic calculation based on the mean-field approximation from statistical physics and demonstrate through numerical experiments the ability of untrained GCNs to detect the presence of clusters and to recover the ground truth community assignments of vertices in the stochastic block model. They empirically show that the regime of graph model parameters in which an untrained GCN is successful at this task agrees well with the analytically derived detection threshold. The authors also conjecture that training GCNs does not significantly improve their community detection performance. }

The theory of graphons as limits of dense graph sequences was
initiated in~\cite{LovaszGraphLimits} and developed by various
authors~\cite{BorgsChayes2008,BorgsChayesII}.
For a comprehensive treatment of graph limit theory, see~\cite{lovaszbook}.
Several authors have investigated the problem of estimation of
graphons from samples~\cite{ChanAiroldi2014,gao2015,Klopp2019}.  Our work is complementary
to these, as our main focus here is to characterize the capabilities of
GCNs and their hyperparameters, and the problem of distinguishing graphons
is only a convenient tool for doing so.

\subsection{Our Contributions}
We characterize the ability of GCNs of moderate depth to distinguish between pairs
of graphons.  This ability is characterized in terms of the total variation distance between 
their random walk stationary distributions.  This leads to the notion of
\emph{$\delta$-exceptional pairs of graphons} -- that is, pairs of graphons whose vertex degree distributions are separated in norm by at most $\delta$.  In particular, our Theorem~\ref{thm:convergence-result} gives an upper bound on the $L_{\infty}$ distance between graph embedding vectors from two $\delta$-exceptional graphons in terms of $\delta$.  Theorem~\ref{thm:prob-err-lower-bound} then gives a lower bound on the probability of error of \emph{any} test that
attempts to distinguish between two graphons based on slightly perturbed $K$-layer GCN embedding 
matrices of sample graphs of size $n$, provided that $K = \Omega(\log n)$. 
Here, the constant hidden in the $\Omega(\cdot)$ depends on the mixing
times of the random sample graphs drawn from the distributions induced by the two graphons.

We show that our theorems are not vacuous by exhibiting,
in Theorem~\ref{thm:sbm-exceptionality}, a family of 
$0$-exceptional graphon pairs (more specifically, stochastic block model
pairs) and show that their limiting GCN embedding vectors converge to 
the same point and are thus indistinguishable.


We then show a converse achievability result in Theorem~\ref{thm:achievability} that says, roughly, that provided that the number of layers is
sufficiently large ($K = \Omega(\log n)$), there exists a \emph{linear} GCN architecture 
with a very simple sequence of weight
matrices and a choice of initial embedding matrix such that $\delta$-separated pairs of graphons (i.e., graphon pairs that are not $\delta$-exceptional) 
are distinguishable based on the noise-perturbed GCN embeddings of their sample
graphs.  In other words, this indicates that the family of graphon pairs in
the previous theorems is the \emph{only} family for which
distinguishability by GCNs is impossible.

Our proofs rely on concentration of measure arguments and techniques from the theory
of Markov chain mixing times and spectral graph theory~\cite{LevinPeresWilmer2006}.

Finally, we empirically demonstrate our results on synthetic and real
datasets.  The synthetic dataset consists of pairs of stochastic block 
models from the family described in Theorem~\ref{thm:sbm-exceptionality}.
The real dataset is the MUTAG network dataset, which is a set of nitro chemical compounds divided
into two classes~\cite{debnath1991structure} based on their mutagenicity
properties.  We show that the typical degree statistics of the two classes
are hard to distinguish in the sense of $\delta$-separation.  Therefore, our theorems predict that a trained
GCN will perform poorly on the classification task.  Our empirical results demonstrate the
veracity of our theoretical predictions.


\subsubsection{Relations between probability of error lower and upper bounds}
Our probability of classification error lower bounds give theoretical backing to a phenomenon that
has been observed empirically in graph classification problems: adding arbitrarily
many layers (more than $\Theta(\log n)$) to a GCN can substantially degrade 
classification performance.  This is an implication of Theorem~\ref{thm:prob-err-lower-bound}. 
On the other hand, Theorem~\ref{thm:achievability} shows that this is \emph{not} always
the case, and that for \emph{many} pairs of graphons, adding more layers improves
classification performance.  We suspect that the set of pairs of graphons for which
adding arbitrarily many layers does not help forms a set of measure $0$, though this
does not imply that such examples never arise in practice.

The factor that determines whether or not adding layers will improve or degrade performance of a GCN in distinguishing between two graphons $W_0$ and $W_1$ is the distance between the stationary distributions
of the random walks on the sample graphs from $W_0$ and $W_1$. 
This, in turn,
is determined by the normalized degree profiles of the sample graphs.

\section{Notation and Model}
\subsection{Asymptotic notation and norms}
We will
frequently use two particular norms: the $\ell_{\infty}$ norm for vectors and
matrices, which is the maximum absolute entry; and the operator norm induced
by $\ell_{\infty}$ for matrices: for a matrix $M$, 
\begin{align}
    \| M \|_{op,\infty}
    = \sup_{v ~:~ \|v\|_{\infty}=1} \| Mv\|_{\infty}.
\end{align}

We will also use standard Landau asymptotic notation.  I.e., for two functions $f, g:\R\to\R$ or similar,
we say that $f(x) = O(g(x))$ as $x\to x_0$ if there exists a constant $C > 0$ such that
$|f(x)| \leq C|g(x)|$ for all $x$ sufficiently close to $x_0$.  We say that $f(x) = \Omega(g(x))$ if $g(x) = O(f(x))$.  We say that $f(x) = \Theta(g(x))$ if $f = O(g)$ and
$f = \Omega(g)$.  We say that $f(x) = o(g(x))$ if $\lim_{x\to x_0} |f(x)/g(x)| = 0$,
and we say that $f(x) = \omega(g(x))$ if $g(x) = o(f(x))$.

\subsection{Graph Convolutional Networks}
We start by defining the model and relevant notation.  
\ANM{Asymptotic and norm notation.}
A $K$-layer
graph convolutional network (GCN) is a function mapping graphs to vectors over
$\R$.  It is
parameterized by a sequence of $K$
\emph{weight matrices} $W^{(j)} \in \R^{d\times d}$, $j \in \{0, ..., K-1 \}$, where $d \in \N$ is the \emph{embedding dimension}, a hyperparameter.
From an input graph $G$ with adjacency matrix $A$ and random walk matrix $\hat{A}$,
and starting with an initial embedding matrix $\hat{M}^{(0)}$, the $\ell$th
embedding matrix is defined as follows:
\begin{align}
    \hat{M}^{(\ell)}
    =  \sigma(\hat{A} \cdot \hat{M}^{(\ell-1)} \cdot W^{(\ell-1)}),
    \label{GCNRecurrence}
\end{align}
where $\sigma:\R\to\R$ is a fixed nonlinear \emph{activation function}
and is applied element-wise to an input matrix.  An \emph{embedding vector} $\hat{H}^{(\ell)} \in \R^{1\times d}$ is then produced by averaging the rows of $\hat{M}^{(\ell)}$:
\begin{align}
    \hat{H}^{(\ell)}
    = \frac{1}{n} \cdot \mathbf{1}^T \hat{M}^{(\ell)}.
\end{align}

Typical examples of activation functions in neural
network and GCN contexts include
\begin{itemize}
    \item 
        Rectified linear unit (ReLU): $\sigma(x) = x\cdot I[x > 0]$.
    \item
        Sigmoid function: $\sigma(x) = \frac{1}{1+e^{-x}}$.
    \item
        Hyperbolic tangent: $\sigma(x) = \tanh(x)$.
\end{itemize}
Empirical work has given evidence that the performance of GCNs on certain
classification tasks is unaffected by replacing nonlinear activation functions
by the identity~\cite{Wu2019SimplifyingGC}.

Frequently, $\hat{A}$ is replaced by either the normalized adjacency matrix
$D^{-1/2}AD^{-1/2}$, where $D$ is a diagonal matrix with the degrees of the
vertices of the graph on the diagonal, or some variant of the Laplacian matrix 
$D - A$.  For simplicity, we will consider in this paper only the choice of 
$\hat{A}$.

The defining equation (\ref{GCNRecurrence}) has the following interpretation:
multiplication on the left by $\hat{A}$ has the effect of replacing each node's
embedding vector with the average of those of its neighbors.  Multiplication on
the right by the weight matrix $W^{(\ell-1)}$ has the effect of replacing each
coordinate (corresponding to a feature) of each given node embedding vector with a 
linear combination of values of the node's features in the previous layer.

\subsection{Graphons}
In order to probe the ability of GCNs to distinguish random graph models
from samples, we consider the task of distinguishing random graph models
induced by graphons.  A graphon $W$ is a symmetric, Lebesgue-measurable
function from $[0, 1]^2 \to [0, 1]$.  To each graphon is associated
a natural exchangeable random graph model as follows: to generate a graph
on $n$ vertices, one chooses $n$ points $x_1, ..., x_n$ (the \emph{latent vertex positions}) uniformly
at random from $[0, 1]$.  An edge between vertices $i, j$ is independent
of all other edge events and is present with probability $W(x_i, x_j)$.
We use the notation $G \sim W$ to denote that $G$ is a random sample
graph from the model induced by $W$.  The number of vertices will be
clear from context.

One commonly studied class of models that may be defined
equivalently in terms of sampling from graphons is the class of 
stochastic block models.  A stochastic
block model on $n$ vertices with two blocks is parameterized by
four quantities: $k_1, p_1, p_2, q$.  The two blocks of vertices
have sizes $k_1 n$ and $k_2 n = (1-k_1)n$, respectively.  Edges
between two vertices $v, w$ in block $i$, $i \in \{1, 2\}$, 
appear with probability $p_i$, independently of all other edges.
Edges between vertices $v$ in block $1$ and $w$ in block $2$
appear independently with probability $q$.  We will generally
write this model as $\SBM(p_1, p_2, q)$, suppressing $k_1$.

An important metric on graphons is the \emph{cut distance}~\cite{jansongraphons}.  It is induced
by the cut norm, which is defined as follows: fix a graphon $W$.  Then
\begin{align}
    \| W \|_{cut} = \sup_{S,T} \left|\int_{S\times T} W(x, y) ~d\mu(x)~d\mu(y) \right|,
\end{align}
where the supremum is taken over all measurable subsets of $[0,1]$, and the integral is taken with respect to the Lebesgue
measure.  For finite graphs, this translates to taking the pair of subsets $S, T$ of vertices that has the maximum between-subset edge density.
The cut \emph{distance} $d_{cut}(W_0, W_1)$ between graphons $W_0, W_1$ is then defined as
\begin{align}
    d_{cut}(W_0, W_1)
    = \inf_{\phi} \|W_0 - W_1(\phi(\cdot), \phi(\cdot))\|_{cut},
\end{align}
where the infimum is taken over all measure-preserving bijections of $[0, 1]$.  In the case of finite graphs,
this intuitively translates to ignoring vertex labelings.
The cut distance generates the same topology on the space of graphons
as does convergence of subgraph homomorphism densities (i.e., \emph{left convergence}), and so it is an important part of the theory of graph limits.

\subsection{Main Hypothesis Testing Problem}
We may now state the hypothesis testing problem under consideration.
Fix two graphons $W_0, W_1$.  A coin $B \sim \Bernoulli(1/2)$ is flipped,
and then a graph $G \sim W_B$ on $n$ vertices is sampled.  Next, $G$
is passed through $K=K(n)$ layers of a GCN, resulting in a matrix $\hat{M}^{(K)} \in \R^{n\times d}$
whose rows are node embedding vectors.  The graph embedding
vector $\hat{H}^{(K)}$ is then defined to be $\frac{1}{n}\mathbf{1}^T \hat{M}^{(K)}$.
As a final step, the embedding
vector is perturbed in each entry by adding an independent, uniformly
random number in the interval $[-\epsres, \epsres]$, for a parameter
$\epsres > 0$ that may depend on $n$, which we will typically consider to be $\Theta(1/n)$.  This results in a vector
$H^{(K)}$.  We note that this perturbation step has precedent 
in the context of studies on the performance of neural networks in the presence of numerical 
imprecision~\cite{numericalprecisionShanbhag}.  For our purposes, it will allow
us to translate convergence results to lower bounds on the probability of misclassification error.

Our goal is to study the effect of the number of layers $K$ and presence or absence of nonlinearities on the 
representation properties of GCNs and probability of error of optimal tests $\Psi(H^{(K)})$ that are meant to 
estimate $B$.  Throughout, we will consider the case where $d=n$.

\section{Main Results}
\subsection{Notation and definitions}
To state our results, we need a few definitions.
For a graphon $W$, we define the
degree function $d_W:[0, 1]\to \R$ to be
\begin{align}
    d_W(x)  = \int_{0}^1 W(x, y)~dy,
\end{align}
and define the total degree function
\begin{align}
    D(W) = \int_{0}^1 \int_{0}^1 W(x, y) ~dx~dy.
\end{align}
We will assume in what follows that all graphons $W$ have the property that
there is some $\ell > 0$ for which $W(x, y) \geq \ell$ for all $x, y \in[0,1]$.

The following definition is central to our results characterizing the
set of distinguishable pairs of graphons.
\begin{definition}
    For any $\delta \geq 0$, we say that two graphons $W_0, W_1$ are a $\delta$-\emph{exceptional} pair if
    \begin{align}
        \ddeg(W_0, W_1)
        := \int_{0}^1 \left| \frac{d_{W_0}(\phi(x))}{D(W_0)} - \frac{d_{W_1}(x)}{D(W_1)}  \right| ~dx \leq \delta,
    \end{align}
    for some measure-preserving bijection $\phi:[0, 1]\to [0, 1]$.  If a pair of
    graphons is not $\delta$-exceptional, then we say that they are $\delta$-separated.
\end{definition}

We define the following class of activation functions:
\begin{definition}[Nice activation functions]
    We define $\Act$ to be the class of activation functions $\sigma:\R\to\R$ satisfying
    the following conditions:
    \begin{itemize}
        \item
            $\sigma \in C^2$.
        \item
            $\sigma(0) = 0$, $\sigma'(0) = 1$ and $\sigma'(x) \leq 1$ for all $x$.    
    \end{itemize}
    We call this the class of \emph{nice activation functions}.
\end{definition}
For simplicity, in Theorems~\ref{thm:convergence-result} and \ref{thm:prob-err-lower-bound} below, we will consider activations in the above class; however,
some of the conditions may be relaxed without inducing changes to our results:
in particular, we may remove the requirement that $\sigma'(0) = 1$, and
we may relax $\sigma'(x) \leq 1$ for all $x$ to only hold for $x$ in some
constant-length interval around $0$.
This expanded class includes activation functions such as $\sigma(x) = \tanh(x)$
 and the \emph{swish} and \emph{SELU} functions:
 \begin{itemize}
    \item
        swish~\cite{Hendrycks2017BridgingNA}: $\sigma(x) = \frac{x}{1+e^{-x}}$ 
    \item
        SELU~\cite{klambauer2017self}: 
            $\sigma(x) =  I[x \leq 0] (e^{x}-1) + I[x > 0] x$.
            
 \end{itemize}
 
We also make the following stipulation about the parameters of the GCN:
the initial embedding matrices $\hat{M}^{(b,0)}$ (with $b \in \{0, 1\}$) 
and weight matrices
$\{ W^{(j)} \}_{j=0}^{K}$ satisfy
\begin{align}
    \left\| \hat{M}^{(b,0)T} \right\|_{op,\infty} 
    \cdot \prod_{j=0}^K \| W^{(j)T} \|_{op,\infty} \leq C,
\end{align}
and
    $\sum_{j=0}^K \| W^{(j)T} \|_{op,\infty} \leq E$,
for some fixed positive constants $C$ and $E$.
We say that a GCN whose parameters satisfy these bounds is \emph{norm-constrained}.

\subsection{Statement of results}

\ANM{Convergence results}
\begin{theorem}[Convergence of embedding vectors for a large class of
graphons and for a family of nonlinear activations]\label{thm:thm3-3}
    \label{thm:convergence-result}
    Let $W_0, W_1$ denote two $\delta$-exceptional graphons, for some fixed
    $\delta \geq 0$.
    
    Let $K$ satisfy $D\log n < K$, for some large enough constant $D > 0$
    that is a function of $W_0$ and $W_1$.  
    Consider the GCN with $K$
    layers and output embedding matrix $\hat{M}^{(K)}$, with the additional
    properties stated before the theorem.
    %
    Suppose that $\delta > 0$.
    Then there exists a coupling of the random graphs $G^{(0)} \sim W_0, G^{(1)} \sim W_1$, as $n \to \infty$ such that the embedding vectors $\hat{H}^{(0,K)}$ and
    $\hat{H}^{(1,K)}$ satisfy
    \begin{align}
        \| \hat{H}^{(0,K)} - \hat{H}^{(1,K)} \|_{\infty} \leq  \frac{\delta}{n}(1 + O(1/\sqrt{n}))
        \label{expr:embeddings-convergence-bound-delta-nonzero}
    \end{align}
    with high probability.  
    
    If $\delta = 0$, then we have
    \begin{align}
        \| \hat{H}^{(0,K)} - \hat{H}^{(1,K)} \|_{\infty} \leq  O(n^{-3/2 + \text{const}}),
        \label{expr:embeddings-convergence-bound-delta-zero}
    \end{align}
    and for a $1-o(1)$-fraction of coordinates $i$, 
    $
        | \hat{H}^{(0,K)}_{i} - \hat{H}^{(1,K)}_{i} | = O(1/n^2).
    $
\end{theorem}
\begin{remark}
    The convergence bounds (\ref{expr:embeddings-convergence-bound-delta-nonzero})
    and (\ref{expr:embeddings-convergence-bound-delta-zero}) should be interpreted
    in light of the fact that the embedding vectors have entries on the order
    of $\Theta(1/n)$.
\end{remark}

\ANM{Information-theoretic impossibility results}
\begin{theorem}[Probability of error lower bound]
    \label{thm:prob-err-lower-bound}
    \ANM{Make sure that the labels are right!}
    Consider again the setting of Theorem~\ref{thm:convergence-result}.
    Furthermore, 
    suppose that $\epsres > \frac{\delta}{2n}$.  Let $K$ additionally satisfy
    $K \ll n^{1/2 - \epsilon_0}$, for an arbitrarily small fixed $\epsilon_0 > 0$.
    Then there exist two sequences
    $\{ \Gr_{0,n} \}_{n=1}^{\infty}, \{ \Gr_{1,n} \}_{n=1}^{\infty}$
    of random graph models such that
    \begin{itemize}
        \item
            with probability $1$, samples $G_{b,n} \sim \Gr_{b,n}$ converge in cut
            distance to $W_b$,
        \item
            When $\delta > 0$,
            the probability of error of any test in distinguishing between $W_0$
            and $W_1$ based on $H^{(b,K)}$, the $\epsres$-uniform perturbation of
            $\hat{H}^{(b, K)}$, is at least
            \begin{align}
                \left( 1 - \frac{\delta}{2\epsres n} \right)^{n}
                \label{expr:err-prob-delta-nonzero}
            \end{align}
            \ANM{Explain what this means in more detail.}
            
            When $\delta = 0$, the probability of error lower bound becomes
            \begin{align}
                \exp\left( - \frac{\text{const}}{\epsres \cdot n} \right).
                \label{expr:err-prob-delta-zero}
            \end{align}
    \end{itemize}   
\end{theorem}

\begin{remark}
    When $\epsres = \Theta(1/n)$ and $\delta = \Omega(1)$, the error probability lower bound (\ref{expr:err-prob-delta-nonzero}) is exponentially decaying to $0$.  On
    the other hand, when $\epsres \gg 1/n$ and $\delta = \Omega(1)$, it becomes
    $\exp\left( -\frac{\delta}{2\epsres}\right)(1+o(1))$,
    which is $\Theta(1)$, so that $W_0$ and $W_1$ cannot be distinguished with high probability.  
    
    When $\delta = 0$ and $\epsres = \Omega(1/n)$, the probability of error
    lower bound in (\ref{expr:err-prob-delta-zero}) is $\Omega(1)$.
\end{remark}

\ANM{Achievability results}
We next turn to a positive result demonstrating the distinguishing capabilities
of very simple, linear GCNs.
\begin{theorem}[Distinguishability result]\label{thm:thm_3-7}
    Let $W_0, W_1$ denote two $\delta$-separated graphons. 
    Then there exists a test that distinguishes with probability $1 - o(1)$ 
    between graph samples $G\sim W_0$
    and $G\sim W_1$ based on the $\epsres$-perturbed embedding vector from a
    GCN with $K$ layers, identity initial and weight matrices, and ReLU activation functions, provided that
    $K > D\log n$ for a sufficiently large $D$ and that $\epsres \leq \frac{\delta}{2n}$.
    \label{thm:achievability}
\end{theorem}
The above theorem states that distinguishable graphon pairs can be
distinguished by a GCN of moderate depth \emph{without} training.


Next, to demonstrate that the above results are not vacuous, 
we exhibit a family of stochastic block 
models that are difficult to distinguish as a result of being $0$-exceptional and are such that infinitely
many pairs of them have large cut distance.

To define the family of models, we consider the following density parameter
set: we pick a base point $P_* = (p_{*,1}, p_{*,2}, q_*)$ with all
positive numbers and then define
\begin{align*}
    \Param  
    = \left\{ P ~:~ (0, 0, 0) \prec P = P_* + \tau\cdot (\frac{1}{k_1}, \frac{k_1}{k_2^2}, \frac{-1}{k_2}) \preceq (1, 1, 1) \right\},
\end{align*}
where $\preceq$ is the lexicographic partial order, and $\tau \in \R$.
We have defined this parameter family because the corresponding SBMs
all have equal expected degree sequences.  That is, the vertex degree
distributions are identical for such pairs of graphons.

It may be checked that $\delta$ in Theorems~\ref{thm:convergence-result} and \ref{thm:prob-err-lower-bound}
is $0$ for pairs of graphons from $\Param$.  This gives the following result.
\begin{theorem}\label{Thm:distinguish}
    For any pair $W_0, W_1$ from the family of stochastic block models
    parameterized by $\Param$, there exists a $K > D\log n$, for some
    large enough positive constant $D$, such that the following statements
    hold:
    
    \paragraph{Convergence of embedding vectors}
    There is a coupling of the graphs $G^{(0)} \sim W_0$ and $G^{(1)} \sim W_1$, as $n\to\infty$ such that the embedding vectors
    $\hat{H}^{(0,K)}$ and $\hat{H}^{(1,K)}$ satisfy
    \begin{align}
        \| \hat{H}^{(0,K)} - \hat{H}^{(1,K)} \|_{\infty} = O(n^{-3/2 + \text{const}})
    \end{align}
    with probability $1 - e^{-\Theta(n)}$.
    
    \paragraph{Probability of error lower bound}
    Let $K$ additionally satisfy $K \ll n^{1/2 - \epsilon_0}$, for an arbitrary small fixed
    $\epsilon_0 > 0$.  Then there exist two sequences
    $\{ \Gr_{0,n} \}_{n=1}^{\infty}$, 
    $\{ \Gr_{1,n} \}_{n=1}^{\infty}$ of random graph models
    such that
    \begin{itemize}
        \item
            with probability $1$, samples $G_{b,n} \sim \Gr_{b,n}$ converge in cut distance to $W_b$,
        \item
            the probability of error of any test in distinguishing between
            $W_0$ and $W_1$ based on $H^{(b,K)}$, the $\epsres$-uniform
            perturbation of $\hat{H}^{(b,K)}$, is lower
            bounded by
            $
                \exp\left( -\frac{C}{\epsres n} \right).
            $
    \end{itemize}
    \label{thm:sbm-exceptionality}
\end{theorem}

We note that while the set of graphon pairs for which $\delta = 0$
forms a set of measure $0$, such classification problems nonetheless
arise in practice.  We exhibit such a case in our empirical results
section.

\section{Empirical results}
\subsection{Performance on synthetic datasets}
Here we validate Theorem~\ref{Thm:distinguish} by considering GCN classification of the class of parametric stochastic block models (SBMs) through a series of experiments. Different families of $\delta$-separated SBMs are considered in our experiments. First, we consider pairs of SBMs in the class $\Param$ with a given base point, where $\delta = 0$. It is verified that SBMs in this class are not easily distinguishable when the depth of the GCN is excessively large. We generated a database of graphs by randomly sampling from two classes of SBMs with parameters $(p_1,p_2,q)=(.6,.4,.2)$, $\tau=.05$ and $(k_1,k_2)=(.5,.5)$. The complete dataset consists of 200 random graphs (100 per class), each containing 1000 nodes. The dataset was randomly split into training and test sets in the ratio 80\%/20\%. We consider four GCNs, comprised of ReLU non-linearities and two, four, six and eight layers, respectively. The models are trained on the training set until the training error saturates ($\sim 50$ epochs), and then it is validated on the test set.  When the number of layers in the GCN is small (2), the cross-entropy training error converges to zero, indicating that the GCNs are able to differentiate graphs arising from the two different SBM distributions (see Figure~\ref{fig:train_error_del0_GCN}a). The results generalize on the hold-out test class, where too, the cross-entropy loss converges to zero for the two-layered GCN architecture.  However, as predicted by Theorem~\ref{Thm:distinguish}, when the number of layers is increased sufficiently, the cross-entropy loss fails to reduce to $0$.  Additionally, it is shown in Figure~\ref{fig:train_error_del0_GCN}a that the $L^\infty$-norm of the weight matrices remains bounded during the training phase, and thus the boundedness assumption in Theorem~\ref{Thm:distinguish} holds true.
\begin{figure}
    \centering
    \includegraphics[width=0.98\columnwidth]{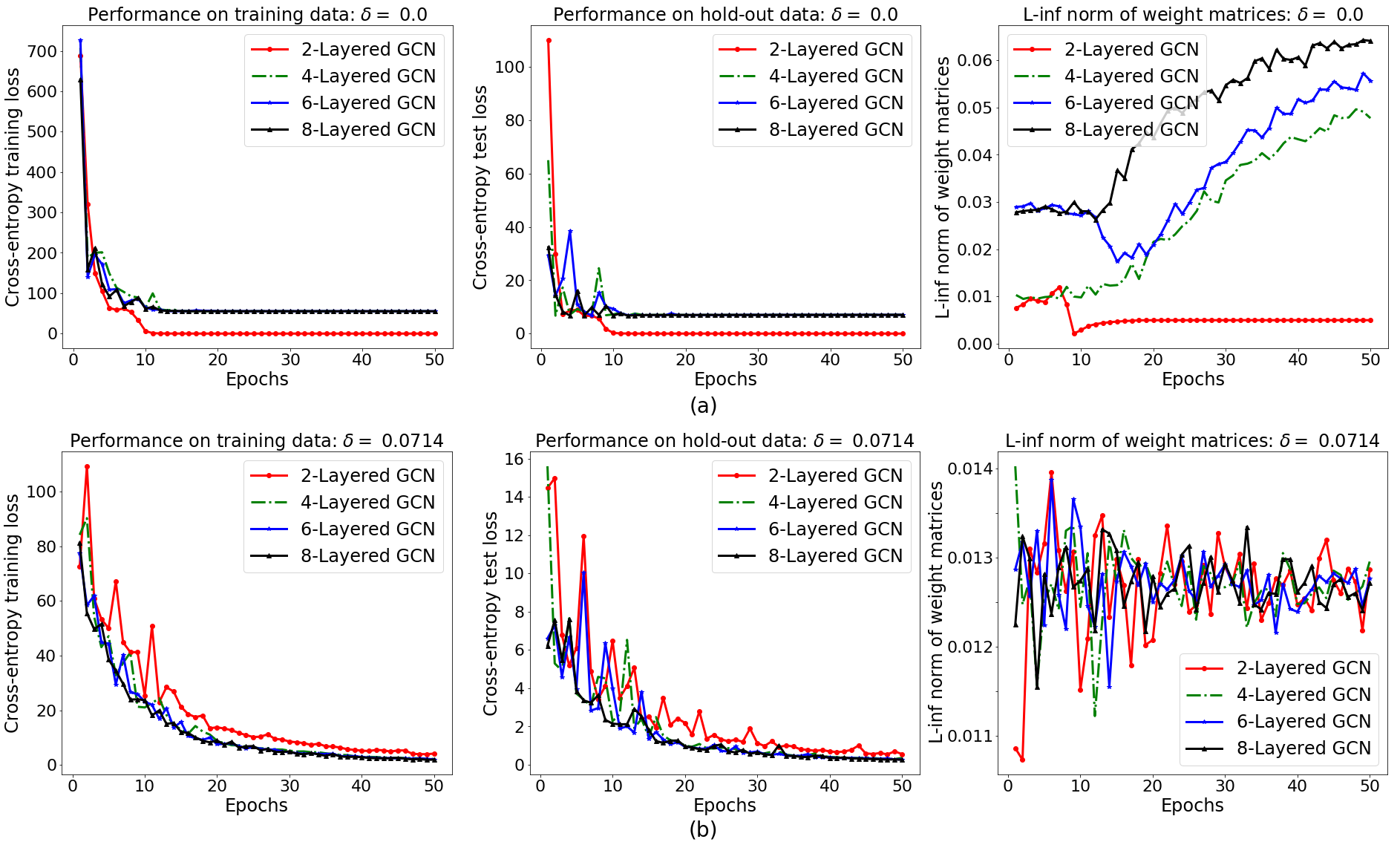}
    \caption{{\small Cross-entropy training and test losses for synthetic SBM datasets in (a) the exceptional class $\Param$, with parameters $(p_1,p_2,q)=(.6,.4,.2)$ and $(p_1',p_2',q')=(.7,.7,.1)$ for GCNs with \emph{trainable} weights for varying number of layers. As expected, distinguishability of GCNs decreases with increase in number of layers, (b) $\delta$-separated class with parameters $(p_1,p_2,q)=(.6,.4,.2)$ and $(p_1',p_2',q')=(.55,.45,.2)$ for \emph{linear} GCNs for varying number of layers. As expected, even linear GCNs with larger number of layers are also able to distinguish graphs from two classes.}}
    \label{fig:train_error_del0_GCN}
\end{figure}

Next we empirically show that if the stationary distribution distance $\delta$ between graphs sampled from two parametric SBMs is large, then the two classes are easily distinguishable by a GCN. Once again, we generated a database of graphs by randomly sampling from two classes of SBMs with parameters $(p_1,p_2,q)=(.6,.4,.2)$ and $(p_1',p_2',q')=(.55,.45,.2)$. The two classes are $\delta$-separated with $\delta=0.0714$. It is shown in Theorem~\ref{thm:thm_3-7} that if $\delta$ is sufficiently large, even a \emph{linear} GCN with identity weight matrices and identity activation functions (equivalent to ReLU in this case) can distinguish graphs from the two classes. Figure~\ref{fig:train_error_del0_GCN}b illustrates the performances of untrained (linear) GCNs on this synthetic dataset. As expected, the performance of GCNs are almost identical regardless of the number of layers, which validates the findings in Theorem~\ref{thm:thm_3-7}.


\subsection{Performance on practical datasets}
Next, we evaluate performance of GCNs on MUTAG, a nitro compounds dataset divided into two mutagenetic classes~\cite{debnath1991structure}. The GCNs have an increasing number of layers. The original dataset consists of 188 nitroaromatic compounds, of which 80 compounds are selected based on the number of nodes. Of these 80 compounds, 53 compounds belong to the `negative' class, and 27 belong to the `positive' class. The number of nodes in the selected graphs are 1067 or higher, which puts us in the large graph regime. As shown in the tSNE plot of the degree distributions for the two classes, and the histogram of average node degrees across two classes in Fig.~\ref{fig:train_error_prac}, it is evident that the two classes are not well separated in terms of degree distributions. Thus increasing the number of layers in GCNs will have adverse effects on distinguishability according to Theorem~\ref{thm:thm3-3}. As before, we consider several GCNs with increasing number of layers - 2, 3, 4 and 5, respectively. As seen in Figure~\ref{fig:train_error_prac}, GCNs with 2 or 3 layers are able to distinguish graphs from two different classes, while cross-entropy loss for GCNs with 4 or 5 layers settles to a very large value.
\begin{figure}
    \centering
    \includegraphics[width=0.98\columnwidth]{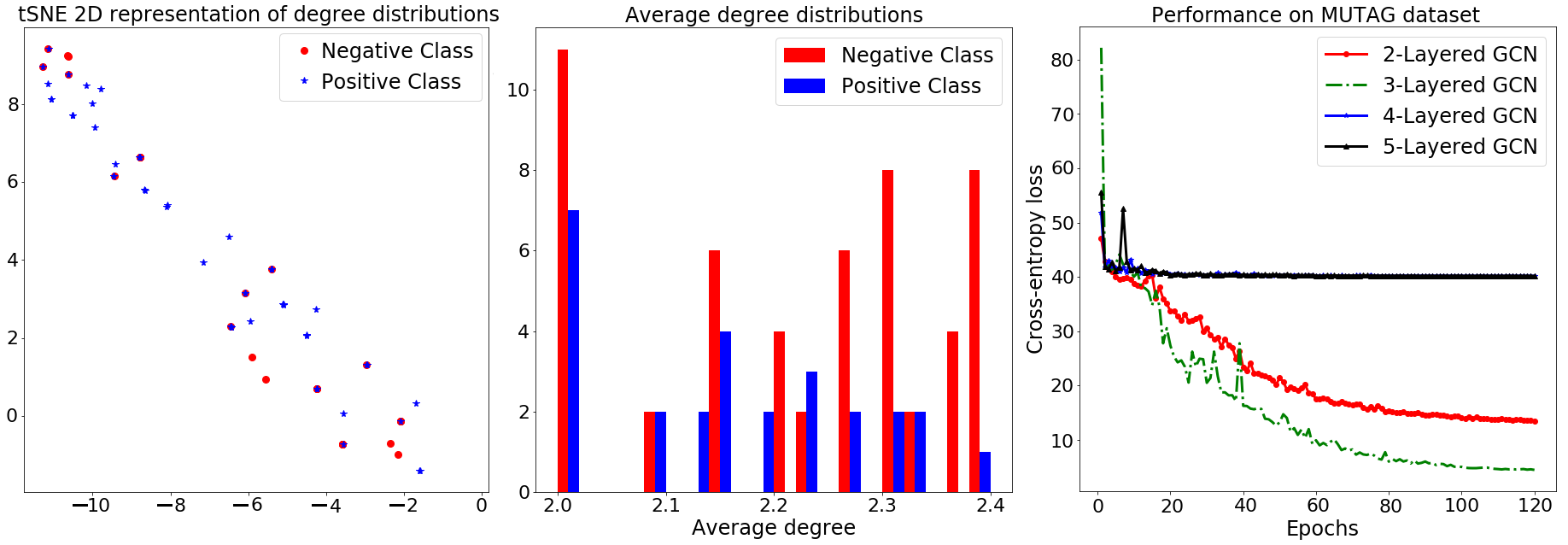}
    \caption{Cross-entropy training loss for predicting nutageneticity of nitroaromatic compounds. As expected, distinguishability of GCNs decreases with increase in number of layers.}
    \label{fig:train_error_prac}
\end{figure}

\ANM{FINISH ME!}

\section{Proofs}

\subsection{Proof of Theorem~\ref{thm:convergence-result}}
We prove this in two parts: we first show that it is true for the case of identity
activation functions.  We then give a result reducing from the case of nonlinear
activations to identity ones.

%
%
%
%

Let us first assume that the activation functions and the initial embedding and weight
matrices are all the identity.  In other words, we first study powers of random walk
matrices.  For a graph $G$ whose random walk chain is ergodic, we define $\pi_{G}$ to
be its stationary distribution.  We define $\hat{A}^{\infty}$ to be the $n\times n$
matrix whose rows are copies of the row vector $\pi_{G}$.  The following lemma translates
$\delta$-exceptionality and $\delta$-separatedness of two graphons to an upper bound on the 
$L_\infty$ distance between their random walk stationary distributions.
\begin{lemma}[$\delta$-exceptionality implies closeness of limit matrices]
    Let $G_0 \sim W_0$ and $G_1 \sim W_1$, where $\ddeg(W_0, W_1)=\delta$.
    There exists a coupling of $G_0, G_1$ such that the following holds:
    If $\delta > 0$,
    \begin{align}
        \| \pi_{G_0} - \pi_{G_1} \|_{\infty}
        = \frac{\delta}{n} \cdot (1 + O(1/\sqrt{n})).
    \end{align}
    If $\delta = 0$, then
    \begin{align}
        \| \pi_{G_0} - \pi_{G_1} \|_{\infty}
        = O(n^{-3/2 + const}),
    \end{align}
    where $const$ is any arbitrarily small constant.
    All of this holds with probability $1 - e^{-\Theta(n)}$.
    \label{lemma:exceptionality-closeness-limit-matrices}
\end{lemma}
\begin{remark}
    In the hypotheses of the lemma, we may alternatively assume an \emph{arbitrary} 
    coupling of $G_0$ and $G_1$, but with the vertices of $G_1$ appropriately 
    reordered.  This is equivalent to assuming that an appropriate measure-preserving
    bijection has been applied to $W_1$ and then coupling $G_0, G_1$ by taking the
    same latent vertex positions for both graphs and, conditioned on these positions,
    generating the graph structures independently of one another.  
\end{remark}
\begin{proof}[Proof of Lemma~\ref{lemma:exceptionality-closeness-limit-matrices}]
    Letting $\pi_{G_0}$ and $\pi_{G_1}$ denote stationary distributions of samples $G_0\sim W_0, G_1 \sim W_1$, and letting the vertices of $G_1$ be reordered so as to minimize the total variation distance between $\pi_{G_0}$ and $\pi_{G_1}$, we have
        \begin{align}
            2d_{TV}(\pi_{G_0}, \pi_{G_1})
            = \sum_{v \in [n]} | \pi_{G_0}(v) - \pi_{G_1}(v) |
            = \sum_{v\in [n]} \left| \frac{\deg_{G_0}(v)}{\sum_w \deg_{G_0}(w)} -   \frac{\deg_{G_1}(v)}{\sum_w \deg_{G_1}(w)} \right|.
        \end{align}
        Now, by the concentration properties of degrees in samples from graphons,
        we have that with probability $1-o(1)$ as $n\to\infty$, 
        \begin{align}
            \frac{\sum_{w} \deg_{G_b}(w)}{n^2} 
            = D(W_b)(1 + O(1/n)).
        \end{align}
        For each vertex $v$, let $x_{b,v} \in [0, 1]$ denote the uniformly randomly
        chosen coordinate corresponding to $v$ in the process of generating $G_b$.
        Then, again by the concentration properties of degrees, and conditioning
        on the values of the $x_{b,v}$ for $b \in \{0, 1\}$ and $v \in [n]$, we have
        that with high probability,
        \begin{align}
            \deg_{G_b}(v) = d_{W_b}(x_{b,v}) n (1 + O(1/\sqrt{n})).
        \end{align}
        By applying a measure-preserving bijection, we may assume that $x_{0,v} = x_{1,v} = x_v$.
        In the case where $\delta > 0$,
        we thus have that
        \begin{align}
        \sum_{v\in [n]} \left| \frac{\deg_{G_0}(v)}{\sum_w \deg_{G_0}(w)} -   \frac{\deg_{G_1}(v)}{\sum_w \deg_{G_1}(w)} \right|
        = \sum_{v\in [n]} \int_{x_v=0}^1  \left| \frac{d_{W_0}(x_v)}{n D(W_0)} - \frac{d_{W_1}(x_v)}{n D(W_1)} \right| (1 + O(1/\sqrt{n})) ~dx_v.
        \label{expr:sum-to-integral}
        \end{align}
        By our assumption that $\ddeg(W_0, W_1)=\delta$, we then have that the inner integral is 
        $\delta/n(1 + O(1/\sqrt{n}))$, so that
        \begin{align}
            \sum_{v\in [n]} \int_{x_v=0}^1  \left| \frac{d_{W_0}(x_v)}{n D(W_0)} - \frac{d_{W_1}(x_v)}{n D(W_1)} \right| (1 + O(1/\sqrt{n})) ~dx_v
            = \delta (1 + O(1/\sqrt{n})).
            \label{dTVLowerBoundAchievability}
        \end{align}
        In the case where $\delta = 0$, cancellations in the approximation of the left-hand side of (\ref{expr:sum-to-integral}) result in
        the estimate
        \begin{align}
            \sum_{v\in [n]} \left| \frac{\deg_{G_0}(v)}{\sum_w \deg_{G_0}(w)} -   \frac{\deg_{G_1}(v)}{\sum_w \deg_{G_1}(w)} \right|           
            = O(n^{-3/2}).
        \end{align}
        
        This establishes that the stationary distributions of the random walks of the
        sample graphs of the two models are bounded away from each other in  total 
        variation distance as a function of $\ddeg(W_0, W_1)$.  Thus, they are also bounded away from each other when viewed
        as vectors in the $L_{\infty}$ distance: in particular, we have that with high probability,
        \begin{align}
            \| \pi_{G_0} - \pi_{G_1}\|_{\infty} = \frac{\delta}{n} \cdot (1 + O(1/\sqrt{n}))
        \end{align}
        when $\delta > 0$ and
        \begin{align}
            \| \pi_{G_0} - \pi_{G_1}\|_{\infty}
            = O(n^{-3/2}).
        \end{align}
        This completes the proof.
\end{proof}
Lemma~\ref{lemma:exceptionality-closeness-limit-matrices} immediately implies that
with probability $1 - e^{-\Theta(n)}$, we have
\begin{align}
    \|\hat{A}^{(0)\infty} - \hat{A}^{(1)\infty} \|_{\infty}
    \leq \frac{\delta}{n} \cdot (1 + O(1/\sqrt{n}))
    \label{expr:limit-matrices-closeness-delta-nonzero}
\end{align}
if $\delta > 0$ and
\begin{align}
    \|\hat{A}^{(0)\infty} - \hat{A}^{(1)\infty} \|_{\infty}
    = O(n^{-3/2 + const}).
    \label{expr:limit-matrices-closeness-delta-zero}
\end{align}

Having established that the limit points of powers of random walk matrices from either
graphon will be close together with high probability, we next give a lemma upper bounding
the distance of a finite power of a random walk matrix to its limit point in terms
of the mixing time of the associated chain.
\begin{lemma}[Distance to limit of random walk matrix powers in terms of mixing times]
    Consider a Markov chain with transition matrix $P$ and stationary
    matrix $P_{\infty}$.  Let $t_{mix}(P, \epsilon)$ denote the
    $\epsilon$-total variation mixing time of $P$.  For any
    $t \geq t_{mix}(P, \epsilon)$, we have that
    \begin{align}
        \|P^t - P_{\infty}\|_{\infty} \leq 2\epsilon.
    \end{align}
    \label{lemma:DistanceToLimitMixingTimeLemma}
\end{lemma}
\begin{proof}
    By definition of $t_{mix}(P, \epsilon)$, whenever $t \geq t_{mix}(P, \epsilon)$, for any initial distribution $\mu_0$ over nodes (a row
    vector),
    \begin{align}
        \frac{1}{2} \|\mu_0P^t - \pi\|_{1} \leq \epsilon.
    \end{align}
    We therefore choose $\mu_0 = e_j^T$ (i.e., the transpose of the
    $j$th standard basis vector.  We note that $\mu_0P^t$ is the
    $j$th row of $P^t$, and the above implies that every element
    of the $j$th row of $P^t$ is within $2\epsilon$ of the corresponding
    element of $\pi$.  This completes the proof.
\end{proof}

Finally, we give a lemma upper bounding the mixing time of the random walk on a sample
graph from either of $W_0, W_1$.
\begin{lemma}[Mixing times of graphons]
    Let $G$ be a sample graph from either $W_0$ or $W_1$.  Then there is a
    positive constant $C$ such that the mixing time of the simple random walk on $G$
    satisfies
    \begin{align}
        \Pr[t_{mix}(G,\epsilon) \geq C\log \frac{n}{\epsilon}]
        \leq e^{-\Theta(n)}.
    \end{align}
    \label{lemma:mixing-time-upper-bound-graphons}
\end{lemma}
\begin{proof}
    We prove this by lower bounding the expansion (also called the 
    bottleneck ratio) of $G$.  Through Cheeger's inequality, this 
    translates to a lower bound on the spectral gap $\gamma_*$ of the random
    walk matrix $P$ (equivalently, an upper bound on the relaxation
    time $t_{rel} = 1/\gamma_*$), which in turn directly upper bounds the mixing time in 
    terms of the minimum stationary probability.
    
    Recall that the bottleneck ratio of a Markov chain with transition
    matrix $P$ with stationary distribution $\pi$ is defined as follows:
    \begin{align}
        \Phi(P) = \min_{S ~:~ \pi(S) \leq 1/2} \frac{Q(S, S^c)}{\pi(S)},
    \end{align}
    where $S$ ranges over subsets of the state space, and $Q(S, S^c)$
    is the probability that, if we start from the stationary distribution
    and take a single step, we move from $S$ to $S^c$.  This
    may be rewritten in the case of random walks on graphs as
    \begin{align}
        \Phi(P)
        = \min_{S ~:~ \pi(S) \leq 1/2} \frac{|\partial S|}{\sum_{x\in S}\deg(x)},
    \end{align}
    where $\partial S$ denotes the \emph{boundary} of $S$, which
    is the set of edges connecting nodes in $S$ and $S^c$.
    Our goal will be to show that with high probability,
    $\Phi(P) > C$, for some positive constant $C$ depending
    on $\Psi$.  This yields the following sequence of implications:
    \begin{align}
        C < \Phi(P)
        \implies C^2/2 < \gamma_*
        \implies t_{rel} \leq 2/C^2.
    \end{align}
    Here, the first implication is from Cheeger's inequality.
    Next, this implies that
    \begin{align}
        t_{mix}(P, \epsilon)
        \leq 2/C^2 \log(\frac{1}{\epsilon\pi_{min}})
        = \Theta(\log \frac{n}{\epsilon}),
    \end{align}
    where we have used the fact that $\pi_{min}$, the minimum stationary probability achieved by any vertex, satisfies $\pi_{min} = O(1/n)$.
    
    We now proceed to lower bound $\Phi(P)$.  A simple
    upper bound on the denominator that holds with probabilit $1$ 
    is as follows:
    \begin{align}
        \sum_{x\in S} 
        \deg(x)
        \leq |S|n.
        \label{DenominatorUpperBound}
    \end{align}
    This follows because the maximum degree of any vertex is $n-1 < n$.
    
    Meanwhile, to lower bound the numerator, we reason as follows.
    Let $S$ be an arbitrary such
    subset of nodes.  We will show that $|S^c|$ must be $\Theta(n)$.  This
    will imply the following lower bound on $\E[|\partial S|]$:
    \begin{align}
        \E[|\partial S|] = \sum_{v \in S} \sum_{w\in S^c} \Pr[\{v, w\} \in G]
        \geq \ell |S| |S^c|
        = \ell |S| \Theta(n)
        = \Theta(|S|n).
    \end{align}
    An application of the Chernoff bound allows us to conclude that
    $|\partial S| = \Theta(|S|n)$ with probability at least $1 - e^{-\Theta(n)}$.  Putting this together with (\ref{DenominatorUpperBound})
    yields the desired lower bound on $\Phi(P)$.
    
    To complete the proof of the lemma, we need to verify that $|S^c| = \Theta(n)$.  Since $\pi(S) \leq 1/2$ and the maximum stationary probability
    of any vertex is, with high probability $\Theta(1/n)$, we have
    \begin{align}
        1/2 \leq \pi(S^c)
        = \sum_{v\in S^c} \pi(v) \leq \frac{c}{n}|S^c|.
    \end{align}
    This implies that $|S^c| = \Omega(n)$, which completes the proof.
    
\end{proof}

We are now ready to complete the proof of Theorem~\ref{thm:convergence-result} for the case 
of identity activations and parameter matrices.  The plan is as follows:
we upper bound the distance between the limit
matrices $\hat{A}^{(b)\infty}$ using Lemma~\ref{lemma:exceptionality-closeness-limit-matrices}.  We then upper
bound the distance between the $K$th embedding matrices and their limits
using Lemmas~\ref{lemma:DistanceToLimitMixingTimeLemma} and
\ref{lemma:mixing-time-upper-bound-graphons}.  Combining these will yield
the desired result.

We have, for some coupling of $G_0$ and $G_1$ (alternatively, in an arbitrary coupling but with the vertices of $G_1$ properly reordered),
\begin{align}
    \|\hat{M}^{(0)K} - \hat{M}^{(1)K}\|_{\infty}
    \leq \| \hat{M}^{(0)K} - \hat{A}^{(0)\infty} \|_{\infty}
    + \| \hat{M}^{(1)K} - \hat{A}^{(1)\infty}  \|_{\infty}
    + \| \hat{A}^{(0)\infty} - \hat{A}^{(1)\infty} \|_{\infty}.
\end{align}
Now, fix an $\epsilon = \epsilon(n) > 0$ to be determined.  By Lemma~\ref{lemma:mixing-time-upper-bound-graphons}, provided that we
choose the number of layers $K \geq D\log n$ for sufficiently large $D$,
and provided that $\epsilon(n)$ is at most polynomially decaying to $0$ as a
function of $n$,
the mixing time of $G_{b} \sim W_b$ satisfies
\begin{align}
    t_{mix}(\hat{A}^{(b)}, \epsilon)
    \leq K.
\end{align}
By Lemma~\ref{lemma:DistanceToLimitMixingTimeLemma}, this implies that
\begin{align}
    \| \hat{M}^{(b)K} - \hat{A}^{(b)\infty} \|_{\infty}
    \leq 2\epsilon.
\end{align}
Finally, (\ref{expr:limit-matrices-closeness-delta-nonzero}) and
(\ref{expr:limit-matrices-closeness-delta-zero}) give an upper bound
on $\| \hat{A}^{(0)\infty} - \hat{A}^{(1)\infty} \|_{\infty}$.
In order to conclude the upper bounds in the theorem statement, we set
$\epsilon(n) = 1/n^2$, so that
\begin{align}
    \| \hat{M}^{(0)K} - \hat{A}^{(0)\infty} \|_{\infty}
    + \| \hat{M}^{(1)K} - \hat{A}^{(1)\infty}  \|_{\infty}
    = O(1/n^2).
\end{align}
This completes the proof in the case of identity parameter matrices.

We next give a lemma that generalizes the above argument to the case of norm-constrained
parameter matrices and activation functions coming from $\Act$.
 \begin{lemma}
    Consider two random walk matrices $\hat{A}^{(0)}$ and
    $\hat{A}^{(1)}$.  
 
    Let $\sigma:\R\to\R$ be in $\Act$.
    
    Furthermore, let the sequence of weight matrices $W^{(0)}, ..., W^{(K)}$ satisfy
    the norm constraints.
   
    Then, if $K \ll n^{1/2-\gamma}$ for arbitrarily small positive constant $\gamma$,
    \begin{align}
        \| \hat{M}^{(0,K)} - \hat{M}^{(1,K)} \|_{\infty}
        \leq \left\| \hat{A}^{(0)K} \hat{M}^{(0,0)} \prod_{j=0}^{K} W^{(j)} - \hat{A}^{(1)K} \hat{M}^{(1,0)} \prod_{j=0}^{K}W^{(j)} \right\|_{\infty} (1 + o(1)).
    \end{align}
    \label{lemma:nonlinearity-lemma}
\end{lemma}
\begin{proof}
    We start by noting that we can approximate $\sigma(x)$ by its first-order Taylor
    expansion:
    \begin{align}
        \sigma(x) = x + \sigma''(\xi) x^2/2
        = x\left(1 + \sigma''(\xi)x/2 \right),
    \end{align}
    where $\xi$ is some real number between $0$ and $x$.  
    
    Furthermore, note that $\sigma''(\xi) = O(x)$.  So we have
    $\sigma(x) = x\left(1 + O(x^2) \right)$.
    
    From this, we have that for each layer $\ell \in [K]$, and for each
    $b \in \{0, 1\}$
    \begin{align}
        \hat{M}^{(b,\ell)}
        &= \sigma(\hat{A}^{(b)} \hat{M}^{(b,\ell-1)} W^{(\ell-1)}) \\
        &= \hat{A}^{(b)} \hat{M}^{(b,\ell-1)} W^{(\ell-1)} 
        \cdot \left( 1 + 
            O\left( \frac{ \| \hat{M}^{(b,\ell-1)T} \|_{op,\infty}^2 \|W^{(\ell-1)T} \|_{op,\infty}^2} {n^2} \right) \right),
        \label{RecurrenceWithActivations}    
    \end{align}
    where the relative error expression comes from the fact that
    \begin{align}
        \| \hat{A}^{(b)} \hat{M}^{(b,\ell-1)} W^{(\ell-1)} \|_{\infty} 
        &\leq \| \hat{A}^{(b)}\|_{\infty} \| \hat{M}^{(b,\ell-1)T} \|_{op,\infty} \|W^{(\ell-1)T} \|_{op,\infty} \\
        &= O\left(\frac{ \| \hat{M}^{(b,\ell-1)T} \|_{op,\infty}^2 \|W^{(\ell-1)T} \|_{op,\infty}^2} {n^2} \right)
    \end{align}    
    
    Iterating the recurrence (\ref{RecurrenceWithActivations}) gives us
    \begin{align}
        \hat{M}^{(b,\ell)}
        = \hat{A}^{(b)\ell} \hat{M}^{(b,0)} \prod_{j=0}^{\ell-1} W^{(j)}
        \cdot \prod_{j=0}^{\ell-1}  \left( 1 + O\left(\frac{ \| \hat{M}^{(b,j)T} \|_{op,\infty}^2 \|W^{(j)T} \|_{op,\infty}^2} {n^2} \right)\right).
    \end{align}
    
    Now, we will show an upper bound on $\| \hat{M}^{(b,j)T} \|_{op,\infty}$.
    We will in particular show that
    \begin{align}
        \| \hat{M}^{(b,j)T} \|_{op,\infty}
        \leq \prod_{i=0}^{j} \|W^{(i)T}\|_{op,\infty}\cdot \left(1 + O(n^{-1/2+\gamma}) \right)^j,
        \label{OperatorNormUpperBound}
    \end{align}
    for arbitrarily small $\gamma > 0$,
    which implies, by our initial assumption, that $\| \hat{M}^{(b,j)T} \|_{op,\infty}
    = O(1)$ as $n\to\infty$.
    To show this, we apply the fact that $|\sigma(x)| \leq |x|$
    for all $x$.  This implies that
    \begin{align}
        \| \hat{M}^{(b,j)T}\|_{op,\infty}
        &\leq \| W^{(j-1)T} \|_{op,\infty} \cdot \| \hat{M}^{(b,j-1)T} \|_{op,\infty} \cdot \|\hat{A}^{(b)T} \|_{op,\infty} \\
        &\leq \left( 1 + O(n^{-1/2+\gamma}) \right)  \| W^{(j-1)T} \|_{op,\infty} \cdot \| \hat{M}^{(b,j-1)T} \|_{op,\infty},
    \end{align}
    with probability $1-e^{-\Theta(n)}$.  Iterating this
    recurrence, we get 
    \begin{align}
        \| \hat{M}^{(b,j)T}\|_{op,\infty}
        \leq \left(1 + O(n^{-1/2+\gamma}) \right)^j \prod_{i=0}^{j-1} \| W^{(j-1)T}\|_{op,\infty}\cdot  \| \hat{M}^{(b,0)T} \|_{op,\infty},
    \end{align}
    as claimed.
    The inequality (\ref{OperatorNormUpperBound}) implies that as long as $\ell \ll n^{1/2-\gamma}$,
    \begin{align}
        \hat{M}^{(b,\ell)}
        = \hat{A}^{(b)\ell} \hat{M}^{(b,0)} \prod_{j=0}^{\ell-1} W^{(j)}
        \cdot \prod_{j=0}^{\ell-1} \left( 1 + O(1/n^2) \right).
    \end{align}
    This implies that
    \begin{align}
        \hat{M}^{(b,\ell)}
        = \hat{A}^{(b)\ell} \hat{M}^{(b,0)} \prod_{j=0}^{\ell-1} W^{(j)}
        \cdot (1 + O(n^{-3/2+\gamma})).       
    \end{align}
    Finally, this implies
    \begin{align}
        \| \hat{M}^{(0,K)} - \hat{M}^{(1,K)} \|_{\infty}
        \leq \left\|  \hat{A}^{(0)K} \hat{M}^{(0,0)} \prod_{j=0}^{K-1} W^{(j)}
          -    \hat{A}^{(1)K} \hat{M}^{(1,0)} \prod_{j=0}^{K-1} W^{(j)}  \right\|_{\infty} (1 + O(n^{-3/2+\gamma})),
    \end{align}
    so that we have reduced the problem to the linear case, as desired.
\end{proof} 

The reduction from the general, norm-constrained linear case to the case
of identity parameter matrices is a simple matter of applying operator
norm subadditivity.

Lemma~\ref{lemma:nonlinearity-lemma}, together with the proof of the simple linear case, 
completes the proof of the theorem.

\subsection{Proof of Theorem~\ref{thm:prob-err-lower-bound}}

%
In order to prove the probability of error lower bound, we will
need to define our models $G_{n,0}, G_{n,1}$ in such a way that they may be coupled,
with the property that almost all vertex degrees be approximately
equal in samples from both models.  We describe our construction as
follows.
\begin{itemize}
    \item
        Generate $G'_{n,0} \sim W_0$, $G'_{n,1}\sim W_1$ independently.
        Set $G_{n,1} = G'_{n,1}$.
    \item
        Fix some $C > 1$.
        In any graph $G$, we say that a vertex $v$ is $C$-\emph{small} if its degree satisfies
        \begin{align}
            \deg_G(v) < \deg_{G'_{n,1}}(v) - C.
        \end{align}
        Similarly, we say that it is $C$-\emph{large} if 
        \begin{align}
            \deg_G(v) > \deg_{G'_{n,1}}(v) + C.
        \end{align}
        All other vertices will be said to be $C$-\emph{just right}.
        Now, we will repair $G'_{n,0}$ so that its degrees are very close to those 
        of $G_{n,1}$.  
        In $G'_{n,0}$, we sequentially add in an arbitrary order all possible edges
        connecting $C$-small vertices, until no more such edges may be added.  Note
        that in the course of doing this, some vertices may cease to be $C$-small.
        We then analogously remove edges between $C$-large vertices.  This results
        in a graph which we call $G_{n,0}$.
    \item
        Finally, we are left with a set of $C$-small and a set of $C$-large vertices in $G_{n,0}$, as well as a set of $C$-just right vertices.  
        We will show in Lemma~\ref{LargeSmallJustRightLemma} below that
        only an asymptotically small number of vertices are $C$-small or $C$-large.
        It will additionally be important that at most $O(n^{1/2+\gamma})$ edges per
        vertex have been modified in transforming $G'_{n,0}$ into $G_{n,0}$.
\end{itemize}
Regarding the above construction, we have the following facts.
\begin{lemma}
    The graphs $G_{n,0}$ and $G_{n,1}$ converge in cut distance to $W_0$ and
    $W_1$, respectively, with probability $1$.
\end{lemma}
\begin{proof}
    This follows directly from the cut distance convergence of 
    $G'_{n,0}$ and $G'_{n,1}$ to $W_0$ and $W_1$, along with the fact that we did not
    add or remove more than $O(n^{1/2+\gamma})$ edges incident on any given
    vertex.  In particular, this means that the density (normalized by $n^2$) 
    of edges between any
    pair of subsets of vertices was not perturbed by more than $O(n^{-1/2+\gamma})$.
\end{proof}

\begin{lemma}[Number of large, small, and just right vertices in $G_{n,0}$]
    With high probability, the numbers of large, small, and $C$-just right vertices
    in $G_{n,0}$ are $O(n^{1/2+\gamma}), O(n^{1/2+\gamma})$, and $n(1 - O(n^{-1/2+\gamma}))$, respectively, where $\gamma > 0$ is arbitrarily small.
    Moreover, all of the large and small vertices are $O(n^{1/2+\gamma})$-just right.
    \label{LargeSmallJustRightLemma}
\end{lemma}
\begin{proof}
    We note that after the edge addition process stops, the set of small
    vertices in $G_{n,0}$ forms a clique.  If this were not so, then we could
    continue by adding at least one edge between small vertices.
    Furthermore, since, with probability exponentially close to $1$ as a function
    of $n$, the discrepancies between the degrees of corresponding vertices in 
    $G'_{n,0}$ and $G'_{n,1}$ are all at most $O(n^{1/2+\gamma})$, for any fixed
    $\gamma > 0$, we know that we added at most $O(n^{1/2+\gamma})$ edges per
    small vertex in $G'_{n,0}$ to get $G_{n,0}$.  In $G'_{n,0}$, with high
    probability, the size of the largest subset of vertices such that adding 
    $O(n^{1/2+\gamma})$ edges per vertex yields a clique is $O(n^{1/2+\gamma})$.  This
    implies that the set of small vertices in $G_{n,0}$ has cardinality $O(n^{1/2+\gamma})$.  The same holds for the set of large vertices. Finally, this implies that the number of just right vertices is 
    $n(1-O(n^{-1/2+\gamma}))$.
\end{proof}

Lemma~\ref{LargeSmallJustRightLemma} will be important in establishing our error probability lower bound.

Le Cam's method is the tool of choice for lower bounding the error
probability of any test that distinguishes between $W_0$ and $W_1$
from a sample graph $G_{n,B}$.  In particular, 
Le Cam's method
requires us to upper bound the following quantity:
\begin{align}
    d_{TV}(H^{(K,0)}, H^{(K,1)}),
    \label{dTVExpression}
\end{align}
the total variation distance between the laws of the perturbed
output embedding vectors under the models corresponding to $W_0$ and $W_1$.

To upper bound the quantity in~(\ref{dTVExpression}), we will need the following lemma, 
which gives an expression for the total
variation distance between the random perturbations of two fixed 
matrices in terms of their $L_\infty$ distance.
\begin{lemma}[Exact expression for $d_{TV}$ of perturbed matrices]
    Suppose that $M^{(0)}$ and $M^{(1)}$ are $m\times n$ matrices, for
    arbitrary positive integers $m, n$.
    Suppose, further, that $\tilde{M}^{(0)}$ and $\tilde{M}^{(1)}$ are independent
    element-wise $\epsres$ perturbations of $M^{(0)}$ and $M^{(1)}$.  Then
    \begin{align}
        d_{TV}(\tilde{M}^{(0)}, \tilde{M}^{(1)})
        &= \frac{\Vol(\Cu(M^{(0)}, \epsres) ~\symmdiff~ \Cu(M^{(1)}, \epsres))}
            {2\Vol(\Cu(M^{(1)}, \epsres))} \\
        &=  \frac{\Vol(\Cu(M^{(0)}, \epsres) ~\symmdiff~ \Cu(M^{(1)}, \epsres))}
            {2\Vol(\Cu(M^{(0)}, \epsres))},  
    \end{align}
    where $\symmdiff$ denotes the symmetric difference between two sets,
    and $\Cu(M, r)$ denotes the axis-aligned hypercube of radius $r$ centered at $M$.
    This can be simplified as follows:
    \begin{align}
        d_{TV}(\tilde{M}^{(0)}, \tilde{M}^{(1)})
        = 1 - \frac{  \prod_{i,j \in [m]\times [n]} \left({2\epsres - |M^{(0)}_{i,j} - M^{(1)}_{i,j}|}\right)_{+}}
            { (2\epsres)^{mn}}.
        \label{dTVSimpleExpression}
    \end{align}
    \label{PerturbedDTVLemma}
\end{lemma}
\begin{proof}
    By definition of total variation distance,
    \begin{align}
        &d_{TV}(\tilde{M}^{(0)}, \tilde{M}^{(1)}) \\
        &= \frac{1}{2} \int_{x \in \Cu(M^{(0)}, \epsres) \lunion \Cu(M^{(1)}, \epsres)}
        \left| \frac{\Ind[ x\in\Cu(M^{(0)},\epsres)  ] }{\Vol(\Cu(M^{(0)},\epsres))} - \frac{ \Ind[x \in\Cu(M^{(1)}, \epsres )  ]  }{\Vol(\Cu(M^{(1)}, \epsres ))} \right| ~dx.
    \end{align}
    and $\tilde{M}^{(1)}$, and making this substitution completes the first part of the proof.  Now, we can provide a closed-form formula as follows.
   
    The volume of the hypercube $\Cu(M^{(0)}, \epsres)$ is given by
    \begin{align}
        \Vol(\Cu(M^{(0)}, \epsres))
        = (2\epsres)^{mn}.
    \end{align}
    
    To compute the volume of the symmetric difference between the two
    hypercubes, we first compute the volume of their intersection,
    then the volume of their union (which is the sum of their volumes,
    minus that of their intersection).  Then the volume of the symmetric
    difference is the volume of the union minus that of the intersection.
    
    The volume of the intersection of the two hypercubes can be computed
    by noting that it is an axis-aligned rectangle, where the length
    along the $(i, j)$ axis is given by
    \begin{align}
        \ell_{i,j} 
        = (2\epsres - |M^{(0)}_{i,j} - M^{(1)}_{i,j} |)_{+},
    \end{align}
    so that the volume of the intersection is
    \begin{align}
        \Vol(\Cu(M^{(0)}, \epsres) \lintersect \Cu(M^{(1)}, \epsres))
        = \prod_{i,j \in [m]\times [n]} \ell_{i,j}.
    \end{align}
 
    This implies
    \begin{align}
        \Vol(\Cu(M^{(0)}, \epsres) ~\symmdiff~ \Cu(M^{(1)}, \epsres))
        = 2\cdot (2\epsres)^{mn} - 2\prod_{i,j \in [m]\times [n]} \ell_{i,j}.
    \end{align}
    
    Finally, we get the following simplified expression for the total
    variation distance:
    \begin{align}
        d_{TV}(\tilde{M}^{(0)}, \tilde{M}^{(1)})
        = 1 - \frac{  \prod_{i,j \in [m]\times [n]} \ell_{i,j}   }
            { (2\epsres)^{mn}}.
    \end{align}
    This completes the proof.
\end{proof}

Now, we use the expression in (\ref{dTVSimpleExpression}) to complete the proof of the 
theorem as follows.  We set $\tilde{M}^{(0)}$ and $\tilde{M}^{(1)}$ in Lemma~\ref{PerturbedDTVLemma} to be $\hat{H}^{(K,0)}$ and $\hat{H}^{(K,1)}$,
respectively.  Thus, in the lemma, $m=1$ and $n$ is the number of vertices in 
the sample graphs.  We need to upper bound the differences between 
$\hat{H}^{(K,0)}_j$ and $\hat{H}^{(K,1)}_{j}$.  To do this, we note that
the stationary distributions $\hat{H}^{(\infty,0)}$ and $\hat{H}^{(\infty,1)}$
of the random walks on $G_{n,0}$ and $G_{n,1}$, respectively, have the following
structure, from Lemma~\ref{LargeSmallJustRightLemma}: for $O(n^{1/2 + \delta})$
indices $j\in[n]$, which correspond to large and small vertices in $G_{n,0}$, we have
\begin{align}
    |\hat{H}^{(\infty,0)}_{j} - \hat{H}^{(\infty,1)}_j|
    = O(n^{1/2+\delta} / n^2)
    = O(n^{-3/2+\delta}).
\end{align}
The remaining $n(1 - O(n^{-1/2+\delta}))$ indices $j$ corresponding
to $C$-just right vertices in $G_{n,0}$ satisfy
$
    |\hat{H}^{(\infty,0)}_{j} - \hat{H}^{(\infty,1)}_j|
    = O(1/n^2).
$
Furthermore, from Lemma~\ref{lemma:DistanceToLimitMixingTimeLemma}
with $\epsilon = O(1/n^2)$, provided that we choose
\newline
$K > \max\{ t_{mix}(\hat{A}^{(0)}, \epsilon), t_{mix}(\hat{A}^{(1)}, \epsilon)\}$, we have that $\hat{H}^{(K,0)}$ and $\hat{H}^{(K,1)}$
have the same structure.  Finally, from Lemma~\ref{lemma:mixing-time-upper-bound-graphons}, and noting that our 
construction of $G_{n,0}$ did not substantially alter its mixing time,
we see that $K$ only needs to be $D\log n$, for some large enough $D$.

Applying Lemma~\ref{PerturbedDTVLemma}, we see that the total variation distance
is upper bounded by
\begin{align}
    &1 - \left( 1 - \frac{C}{2\epsres n^2} \right)^{n} 
    \cdot \left( 1 - O\left( \frac{n^{1/2+\delta}}{\epsres n^2} \right) \right)^{n^{1/2+\delta}} \\
    &= 1 - \left( 1 - \frac{C}{2\epsres n^2} \right)^{n} 
    \cdot \left( 1 - O\left( \frac{n^{-3/2+\delta}}{\epsres} \right) \right)^{n^{1/2+\delta}} \\
    &= 1 - \exp\left(-\frac{C}{2\epsres n}\right)(1 + o(1)),
\end{align}
which is strictly less than $1$, as desired.  
We note that by a standard
fact about the total variation distance, this translates to a lower bound
of $e^{-\frac{C}{2\epsres n}}(1 + o(1))$ for the optimal error probability.

\subsection{Proof of Theorem~\ref{thm:achievability}}
\ANM{EDITED UP TO HERE!}
\ANM{FINISH ME!  The achievability result follows by invoking Lemma~\ref{lemma:exceptionality-closeness-limit-matrices} (for the separation of the limit matrices),
followed by Lemma~\ref{lemma:DistanceToLimitMixingTimeLemma} and Lemma~\ref{lemma:mixing-time-upper-bound-graphons}.}

To prove this, we will show that because $W_0$ and $W_1$ are 
$\delta$-separated, the (unperturbed) output embedding vectors of a GCN with sufficiently
many layers and with all parameter matrices equal to the identity will be bounded away from one another in the $L_{\infty}$ norm.
This will allow us to lower bound the total variation distance between
the distributions of the perturbed output embedding vectors, which implies
the existence of a test distinguishing between $W_0$ and $W_1$ that succeeds
with high probability.

In particular, let $G_0 \sim W_0, G_1 \sim W_1$ under the coupling guaranteed by Lemma~\ref{lemma:exceptionality-closeness-limit-matrices}. 
Then Lemma~\ref{lemma:exceptionality-closeness-limit-matrices}
implies that 
\begin{align}
    \| \hat{A}^{(0)\infty} - \hat{A}^{(1)\infty}\|_{\infty}
    \geq \frac{\delta}{n} \left(1 + O(1/\sqrt{n}) \right).
\end{align}
Furthermore, if the number of layers $K$ of the GCN is at least
$D\log n \geq t_{mix}(\hat{A}^{(b)}, 1/n^2)$, then with high probability,
\begin{align}
    \| \hat{A}^{(b)K} - \hat{A}^{(b)\infty} \|_{1}
    \leq 1/n^2 \implies
    \| \hat{A}^{(b)K} - \hat{A}^{(b)\infty} \|_{\infty}
    \leq 1/n^2.
\end{align}
We thus have that
\begin{align}
    \| \One^T \hat{A}^{(0)K}/n - \One^T \hat{A}^{(1)K}/n \|_{\infty} \geq \delta/n + O(1/n^{3/2}).
\end{align}
This implies that if $\epsres < \frac{\delta}{2n}$, with probability $1-o(1)$, the
two output embedding vectors can be distinguished.  This completes the proof.


\subsection{Proof of Theorem~\ref{thm:sbm-exceptionality}}

This follows simply from the fact that all pairs of graphons parameterized
by $\Param$ are $0$-exceptional and that they satisfy the hypotheses of
Theorems~\ref{thm:convergence-result} and \ref{thm:prob-err-lower-bound}.

\section{Conclusions and future work}
We have shown conditions under which GCNs are fundamentally
capable/incapable of distinguishing between sufficiently well-
separated graphons.

It is worthwhile to discuss what lies ahead for the theory of
graph representation learning in relation to the problem of distinguishing
distributions on graphs.  As the present paper is a first step, we have left
several directions for future exploration.  
Most immediately, although we have proven impossibility results for GCNs with
nonlinear activation functions, we lack a complete
understanding of the benefits of more general
ways of incorporating nonlinearity.  We have shown
that architectures with too many layers cannot reliably
be used to distinguish between graphons coming from a certain exceptional
class.  It would be of interest to determine if more general ways of incorporating
nonlinearity are 
able to generically distinguish between any sufficiently well-separated pair of 
graphons, whether or not they come from the exceptional class.  To this end, we are exploring
results indicating that replacing the random walk matrix $\hat{A}$ in the GCN
architecture with the transition matrix of a related Markov chain with the same
graph structure as the input graph $G$ results in a linear GCN that
is capable of distinguishing graphons generically.

Furthermore, a clear understanding of the role played by the embedding dimension
would be of interest.  In particular, we suspect that decreasing the embedding
dimension results in worse graphon discrimination performance.  
Moreover, a more precise
understanding of how performance parameters scale with the embedding dimension
would be valuable in GCN design.

Finally, we note that in many application domains, graphs are typically sparse.
Thus, we intend to generalize our theory to the sparse graph setting by replacing
graphons, which inherently generate dense graphs, with suitable nonparametric
sparse graph models, e.g., as captured by \emph{graphexes}.

\bibliographystyle{IEEEtran}
\bibliography{gnn-bib}

\end{document}